\documentclass[11pt,reqno,final]{article}

\usepackage{amsmath,amsfonts,amssymb,amsthm,version}
\usepackage{mathrsfs,fancybox,pifont}
\usepackage{graphicx}
\usepackage{booktabs}
\usepackage{algorithm,algorithmic}
\usepackage{cite}
\usepackage{authblk}
\usepackage{natbib}
\usepackage{setspace}
\usepackage{url,hyperref}
\usepackage{fullpage} 
\usepackage[notcite,notref]{showkeys}
\usepackage{color}
\usepackage{subfigure,multirow}
\usepackage{epstopdf}
\usepackage{cases}
\usepackage{geometry}
\usepackage{mathtools}
\usepackage{authblk}
\usepackage{bm} 
\usepackage{lipsum}
\usepackage{indentfirst}
\usepackage[table]{xcolor} 
\usepackage{colortbl}      
\allowdisplaybreaks

\numberwithin{equation}{section}

\theoremstyle{plain}
\newtheorem{proposition}{Proposition}

\theoremstyle{definition}

\theoremstyle{remark}
\newtheorem{remark}{Remark}

\title{AB-Cache: Training-Free Acceleration of Diffusion Models via Adams-Bashforth Cached Feature Reuse}
\author[1]{Zichao Yu \thanks{ zichaoyu@mail.ustc.edu.cn}}
\author[1]{Zhen Zou}
\author[2]{Guojiang Shao \thanks{gjshao@m.fudan.edu.cn}}
\author[1]{Chengwei Zhang}
\author[3]{Shengze Xu}
\author[1]{Jie Huang}
\author[1]{Feng Zhao}
\author[4]{Xiaodong Cun}
\author[1]{Wenyi Zhang\thanks{wenyizha@ustc.edu.cn}}
\vfill
\affil[1]{University of Science and Technology of China, Hefei, China}
\affil[2]{Fudan University, Shanghai, China}
\affil[3]{The Chinese university of HongKong, HongKong, China}
\affil[4]{Great Bay University, Dongguan, China}

\date{}

\begin{document}

\maketitle

\begin{abstract}
Diffusion models have demonstrated remarkable success in generative tasks, yet their iterative denoising process results in slow inference, limiting their practicality. While existing acceleration methods exploit the well-known U-shaped similarity pattern between adjacent steps through caching mechanisms, they lack theoretical foundation and rely on simplistic computation reuse, often leading to performance degradation. In this work, we provide a theoretical understanding by analyzing the denoising process through the second-order Adams-Bashforth method, revealing a linear relationship between the outputs of consecutive steps. This analysis explains why the outputs of adjacent steps exhibit a U-shaped pattern. Furthermore, extending Adams-Bashforth method to higher order, we propose a novel caching-based acceleration approach for diffusion models, instead of directly reusing cached results, with a truncation error bound of only \(O(h^k)\) where $h$ is the step size.   
Extensive validation across diverse image and video diffusion models (including HunyuanVideo and FLUX.1-dev) with various schedulers demonstrates our method's effectiveness in achieving nearly $3\times$ speedup while maintaining original performance levels, offering a practical real-time solution without compromising generation quality.
\end{abstract}

\medskip
  \noindent{\bf Keywords}:Adams-Bashforth method, Caching-based acceleration approach, Diffusion model, Linear approximation relationship, Training-free

\section{Introduction}  
In recent years, diffusion models~\citep{ho2020denoising,song2019generative,song2020score} have emerged as a powerful framework in generative tasks, owing to their exceptional ability to generate high-quality, diverse outputs with strong theoretical underpinnings. These models excel in capturing intricate data distributions~\citep{song2021maximum}, making them highly effective for applications such as image synthesis~\citep{chen2023pixart,rombach2022high}, audio generation~\citep{liu2024audioldm,tian2025audiox}, and text-to-video tasks~\citep{yang2024cogvideox,kong2024hunyuanvideo,hacohen2024ltx}. Architecturally, diffusion models have evolved significantly, transitioning from the foundational U-Net backbone to the more advanced Diffusion Transformer (DiT)~\citep{peebles2023scalable}, which incorporates transformer-based architectures~\citep{vaswani2017attention} for improved scalability and performance. Notable examples of their application include
DALL·E 2 and Stable Diffusion for efficient and photorealistic text-to-image synthesis. However, despite these strengths, a key drawback of diffusion models is their slow inference speed, which stems from the iterative denoising process, limiting their practicality for large-scale, real-time applications.

To address this challenge, recent research has broadly categorized methods for accelerating diffusion model inference into two main approaches: reducing the number of sampling steps and lowering the inference cost per sampling step. The first category focuses on designing advanced probabilistic flow Ordinary Differential Equation (ODE) solvers or employing distillation techniques to condense the denoising trajectory into fewer steps. Methods in the second branch primarily focus on compressing the model size~\citep{kim2023architectural} or utilizing low-precision data formats~\citep{li2023q}.
A novel approach for dynamically adjusting the computation process in diffusion models introduces a specialized caching mechanism during the denoising process~\citep{ma2024deepcache,li2023faster}. This method capitalizes on the high similarity between consecutive steps and the unique properties of U-Net or DiT to cache intermediate computations, which are then efficiently reused in subsequent steps. Alternatively, certain dynamic inference strategies utilize a range of diffusion models, assigning specialized networks to individual steps~\citep{li2023autodiffusion,liu2023oms}.

Despite the proven effectiveness of caching-based acceleration methods for diffusion model inference, two critical concerns remain: 1) The high similarity of certain features observed during the denoising process is primarily experimental, lacking solid theoretical analysis. 2) Current caching mechanisms mainly focus on identifying which modules exhibit higher feature similarity or determining the specific steps where caching should be applied, followed by simply reusing the previously cached computation results. To close this gap, this paper addresses both concerns by providing a theoretical foundation for the observed feature similarities and proposing a simpler but efficient caching mechanism that goes beyond simply reusing previously cached computation results.  

Current literature on caching-based acceleration for diffusion models has experimentally observed a U-shaped phenomenon during the inference process. Unlike previous works~\citep{shen2024lazydit}, which analyze the lower bound of similarity between outputs at consecutive steps and highlight its notably high value, we adopt a numerical integration-based approach. By examining the integral equations underlying the denoising process in diffusion models, we provide a fresh perspective to address this phenomenon. Rather than focusing on the properties of local modules within the network, we treat the network as a whole. Leveraging the second-order Adams-Bashforth method, we discover that the outputs of adjacent steps not only exhibit strong similarity but also follow a linear relationship. Inspired by this insight, we propose a novel yet simple and effective method to accelerate the inference process of diffusion models in a training-free manner, independent of the network architecture. Specifically, by introducing a caching mechanism during the denoising process, rather than merely reusing cached computation results, we perform a binomial combination of the cached results from the previous $k$ steps. This approach achieves significant acceleration through a straightforward linear operation, all without compromising the quality of the generated outputs, with a truncation error of only $O(h^k)$, where $h$ is the
step size.  

We evaluate the effectiveness of our method across a wide range of diffusion models, spanning from image diffusion models to video diffusion models. Experimental results demonstrate that our method achieves significant improvements in inference speed while maintaining comparable performance to the base model without acceleration.

In summary, our contributions are as follows:
\begin{itemize}
    \item Our method establishes a theoretical foundation for the observed feature similarities in diffusion models. By employing numerical integration techniques, we not only confirm the high similarity between adjacent steps but also reveal a linear relationship governing their outputs during the denoising process.
    \item We propose a training-free and architecture-agnostic approach to accelerate diffusion model inference. By leveraging a caching mechanism and a simple linear operation, our method significantly reduces computational costs while maintaining high-quality generation.
    \item Our method is extensively evaluated on a wide range of diffusion models, achieving state-of-the-art results in inference speed without sacrificing quality.
\end{itemize}

\section{Related Works}
Diffusion probability models~\citep{ho2020denoising,song2019generative,song2020score} have achieved remarkable success in generative tasks, initially excelling in image synthesis~\citep{rombach2022high,chen2023pixart,ramesh2022hierarchical} through their ability to model complex data distributions via a gradual denoising process. Early architectures, such as UNet~\citep{ronneberger2015u}, played a pivotal role in this success, leveraging encoder-decoder structures and skip connections to capture both local and global features, which proved essential for generating high-quality images. 
However, the intrinsic scalability limitations of the UNet architecture have posed challenges for large-scale and high-performance models. 
To address these challenges, the Diffusion Transformer (DiT)~\citep{peebles2023scalable} architecture has emerged as a promising solution. DiT harnesses the strengths of transformer-based~\citep{vaswani2017attention} models, renowned for their scalability and ability to handle large-scale data, significantly enhancing the potential of diffusion models~\citep{yang2024cogvideox,zheng2024open,ma2024latte}.


\begin{figure*}[htbp]
    \centering
    \includegraphics[width=\textwidth]{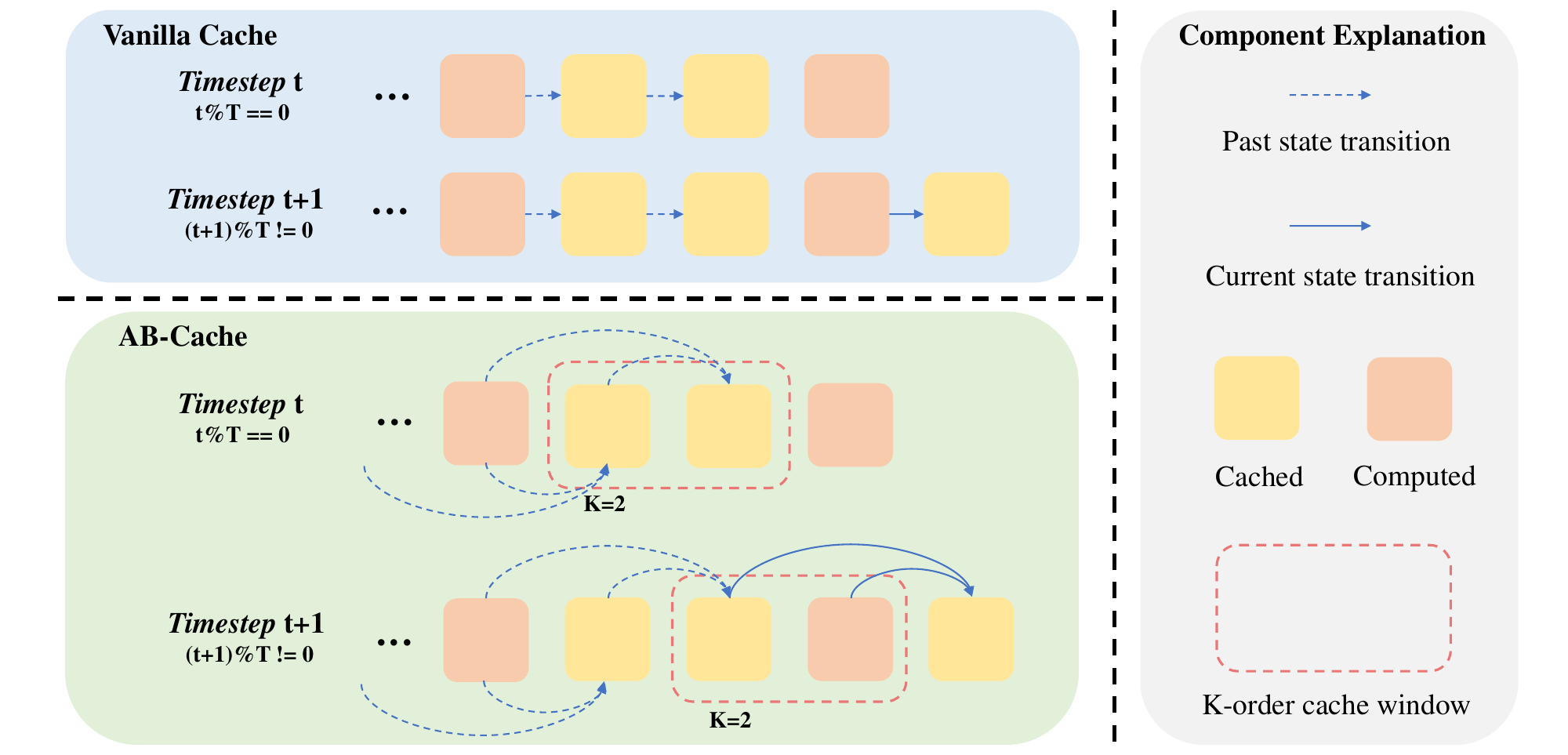} 
    \caption{Similar to the base model, our method computes at timesteps where  \( t\%T == 0 \) and caches states otherwise. However, we replace simple assignment with $k$-th order Adams-Bashforth numerical integration within the cache window, achieving more accurate state transitions while maintaining efficiency, as demonstrated in the K=2 case.}
    \label{fig:pipeline}
\end{figure*}
Despite the remarkable success of diffusion models in image and video generation, their widespread application has been significantly hindered by the substantial computational costs associated with inference. Current research on accelerating diffusion models primarily focuses on two key approaches: reducing the number of sampling steps and minimizing the computational cost of each sampling step.

\textbf{Reduction of Sampling Steps.}The core idea behind this line of work is to maintain high generation quality while significantly reducing the number of sampling steps.  
Many approaches achieve this by employing efficient Stochastic Differential Equation (SDE) or ODE solvers. For instance, DDIM~\citep{song2020denoising} introduced a non-Markovian process that enables faster sampling with fewer steps, while DPM-Solver~\citep{lu2022dpm,lu2022dpmpp,karras2022elucidating} optimized the sampling trajectory using advanced numerical methods. Another direction involves  
distilling the ODE trajectory of the diffusion
sampling process into another neural network that enables
fewer-step sampling. Methods like Progressive Distillation~\citep{salimans2022progressive} and its variants~\citep{song2023consistency,liu2023instaflow,sauer2024adversarial,meng2023distillation,geng2023one} compress the denoising process into fewer steps, significantly reducing inference time. 

\textbf{Reducing Per-Step Cost.}In addition to reducing the number of sampling steps, minimizing the computational cost of each sampling step is another promising approach to accelerating diffusion model inference. These methods can be broadly categorized into two classes: traditional model compression and quantization techniques, and caching-based acceleration methods for diffusion inference.  

\textbf{Model Compression and Quantization.}Model compression techniques, such as pruning~\citep{dong2017learning,lee2019signal,liu2021group}, reduce computational overhead by eliminating less important parts of the model, thereby decreasing inference latency. Similarly, quantization~\citep{bhalgat2020lsq+,li2023q,shang2023post,gu2022vector} achieves this by representing model weights and activations with lower bit precision. However, these methods typically require additional training efforts to fine-tune the model and maintain the generation quality of diffusion models. While effective, this extra training overhead makes it challenging to apply these techniques to large, state-of-the-art models, especially in scenarios with limited computational resources. As a result, alternative approaches, such as caching-based acceleration or architectural optimizations, are often explored to achieve efficient inference without extensive retraining.

\textbf{Caching-based Acceleration.}
Caching-Based Acceleration, a particularly relevant line of work exploits temporal redundancy in the diffusion process by reusing intermediate features between adjacent timesteps. Methods like TGATE~\citep{zhang2024cross} cache and reuse attention outputs at scheduled timesteps, while DeepCache~\citep{ma2024deepcache} and Faster Diffusion~\citep{li2023faster} leverage feature caching mechanisms to indirectly modify the UNet diffusion process, thereby achieving enhanced acceleration.  
Building on these advancements, recent efforts have extended caching-based acceleration to Diffusion Transformer (DiT). For instance, 
$\Delta$-DiT~\citep{chen2024delta} caches the residuals between DiT
blocks, and PAB~\citep{zhao2024real} caches and broadcasts
intermediate computations at different timestep intervals based on the characteristics of varying attention
output.
FasterCache~\citep{lv2024fastercache} accelerates diffusion-based video generation by leveraging a Dynamic Feature Reuse Strategy and exploiting redundancy in the classifier-free guidance (CFG).
These approaches align closely with our work, as they leverage temporal redundancy to reduce computational overhead.

\section{Methodology}

\subsection{Preliminary}
In this subsection, we first review the fundamental concepts of diffusion probability models and their reverse processes.  

\textbf{Diffuison Model.}
Diffusion models are a class of generative models consisting of a forward diffusion process and a reverse generative process. The forward process progressively adds Gaussian noise to a data sample \( \bm{x}_0 \sim p_{\text{data}}(\bm{x}_0) \), eventually destroying its structure. This process is formulated as:
\begin{equation*}
q(\bm{x}_t | \bm{x}_0) = \mathcal{N}(\bm{x}_t; \sqrt{\alpha_t} \bm{x}_0, \sqrt{1 - \alpha_t} I),
\end{equation*}
where \( \alpha_t \) represents a time-dependent noise schedule. Over time, this transformation converts the data distribution into a Gaussian distribution as \( t \to T \).

Regarding the reverse process, given two time steps $s$ and $t$, where $s > 0$ and $t < s$, $x_t$ is calculated as~\citep{lu2022dpm}:
\begin{equation}\label{eq:int}
    \bm{x}_t=\frac{\alpha_t}{\alpha_s} \bm{x}_s-\alpha_t \int_{\lambda_s}^{\lambda_t} e^{-\lambda} \hat{\bm{\epsilon}}_\theta\left(\bm{x}_{t_\lambda(\lambda)}, t_\lambda(\lambda) \right) \mathrm{d} \lambda,
\end{equation}
where $\lambda_t :=\log \left(\alpha_t / \sigma_t\right)$, $t_\lambda(\lambda)$ is the inverse function of $\lambda_t$ satisfying  $t_\lambda(\lambda_t) = t$ and $\bm{\epsilon}_\theta\left(\cdot\right)$ often represents the learned model, which, in our case, is the diffusion transformer. Previous methods show that this integral term can be approximated by adopting Taylor expansion at $\lambda_s$, adopting the first-order or higher-order approximation~\citep{lu2022dpm}. Take the first-order one as an example, the update of $\bm{x}_t$ would be:
\begin{equation}  \label{eq:first-order}
    \bm{x}_{t}=\frac{\alpha_{t}}{\alpha_{s}} \bm{x}_{s}-\sigma_{t}\left(e^{\lambda_{t}-\lambda_{s}}-1\right) \boldsymbol{\epsilon}_\theta\left(\bm{x}_s, s\right).
\end{equation}

\textbf{Adams-Bashforth Method.}
The Adams-Bashforth method ~\citep{butcher2016numerical,hairer1993sp} is a class of explicit linear multistep numerical techniques designed for solving ODEs. Introduced in the 19th century by John Couch Adams and Francis Bashforth, this method employs an iterative extrapolation strategy that leverages prior solution points to predict future values. It is particularly effective for non-stiff ODEs, offering superior accuracy compared to single-step methods such as Euler's approach, while significantly reducing computational overhead by minimizing redundant evaluations. A key advantage of the Adams-Bashforth method is its reliance on zero-order information—specifically, network outputs—without requiring higher-order derivatives or Taylor series expansions. This characteristic makes it both computationally efficient and straightforward to integrate into standard caching frameworks. In practice, lower-order variants of the method are predominantly used due to their simplicity and sufficient accuracy for most applications. As such, the subsequent discussion will focus primarily on these lower-order implementations, which are well-documented and widely referenced in numerical analysis literature.

\subsection{Motivation: Rethinking Similarity in Diffusion Denoising Process}
The denoising process in diffusion models is inherently iterative, requiring multiple steps to transform noisy inputs into high-quality outputs. A key observation in recent works~\citep{ma2024deepcache,zhao2024real} is the high similarity between intermediate computations across adjacent denoising steps, which has been empirically exploited to accelerate inference through caching mechanisms.  Typically, this similarity is characterized by metrics that measure absolute differences, such as the relative $L_p$ norm distance between feature vectors. So, it is natural to ask about the relationship between  
\(\epsilon_{\theta}(\bm{x}_t, t)\) and \(\epsilon_{\theta}(\bm{x}_s, s)\) for adjacent time steps 
\(t\) and 
\(s\). Why do they perform so similarly?
In this section, we provide a theoretical analysis to explore their relationship. 
Before presenting our conclusions, we conducted two preliminary experiments to validate the similarity between adjacent steps: one focusing on measuring value-based similarity, and the other on assessing directional similarity.

Our experiments reveal a more nuanced phenomenon: the similarity between steps is not only reflected in the magnitude of feature values but also in their directional alignment, as illustrated in Fig. \ref{fig:compare}. This directional similarity directly challenges the assumption that simply reusing cached computations from adjacent steps is an optimal strategy, as it ignores the structured evolution of features along a consistent trajectory.  

To address this gap, we present a comprehensive theoretical analysis that treats the entire network as a unified entity, independent of specific architectural choices such as U-Net-based or DiT-based structures. This model-agnostic approach allows our analysis to generalize across diverse diffusion models, providing insights into the fundamental mechanisms driving the similarity in network outputs at adjacent time steps.

  \begin{figure}[htbp]
      \centering
      \hspace{0em}
      \begin{minipage}[t]{0.49\textwidth}
      \centering
      \includegraphics[width=18em]{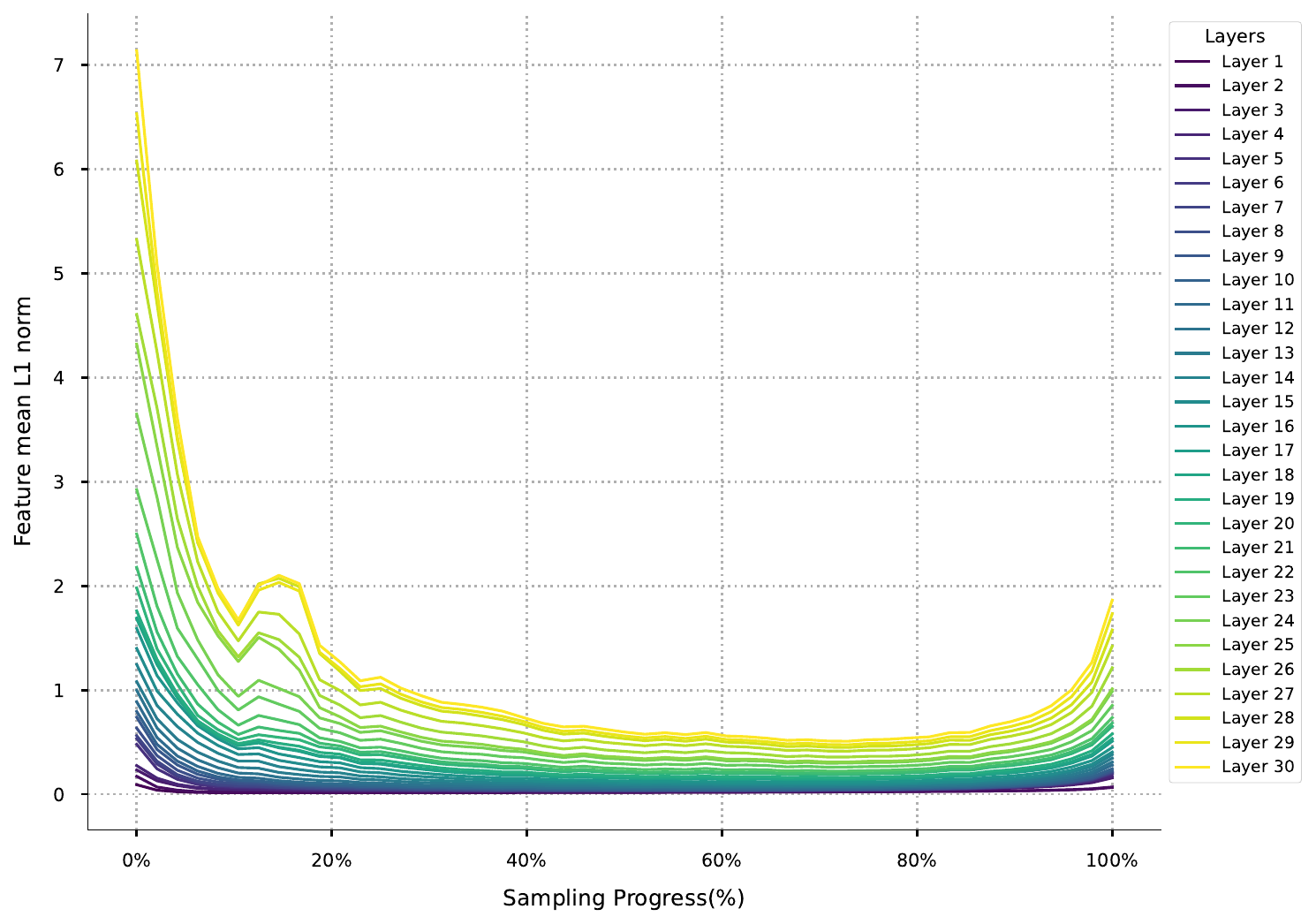}
      \caption{U-type similarity curve.}
      \label{fig:subfig1}
      \end{minipage}
      \begin{minipage}[t]{0.49\textwidth}
      \centering
      \includegraphics[width=18em]{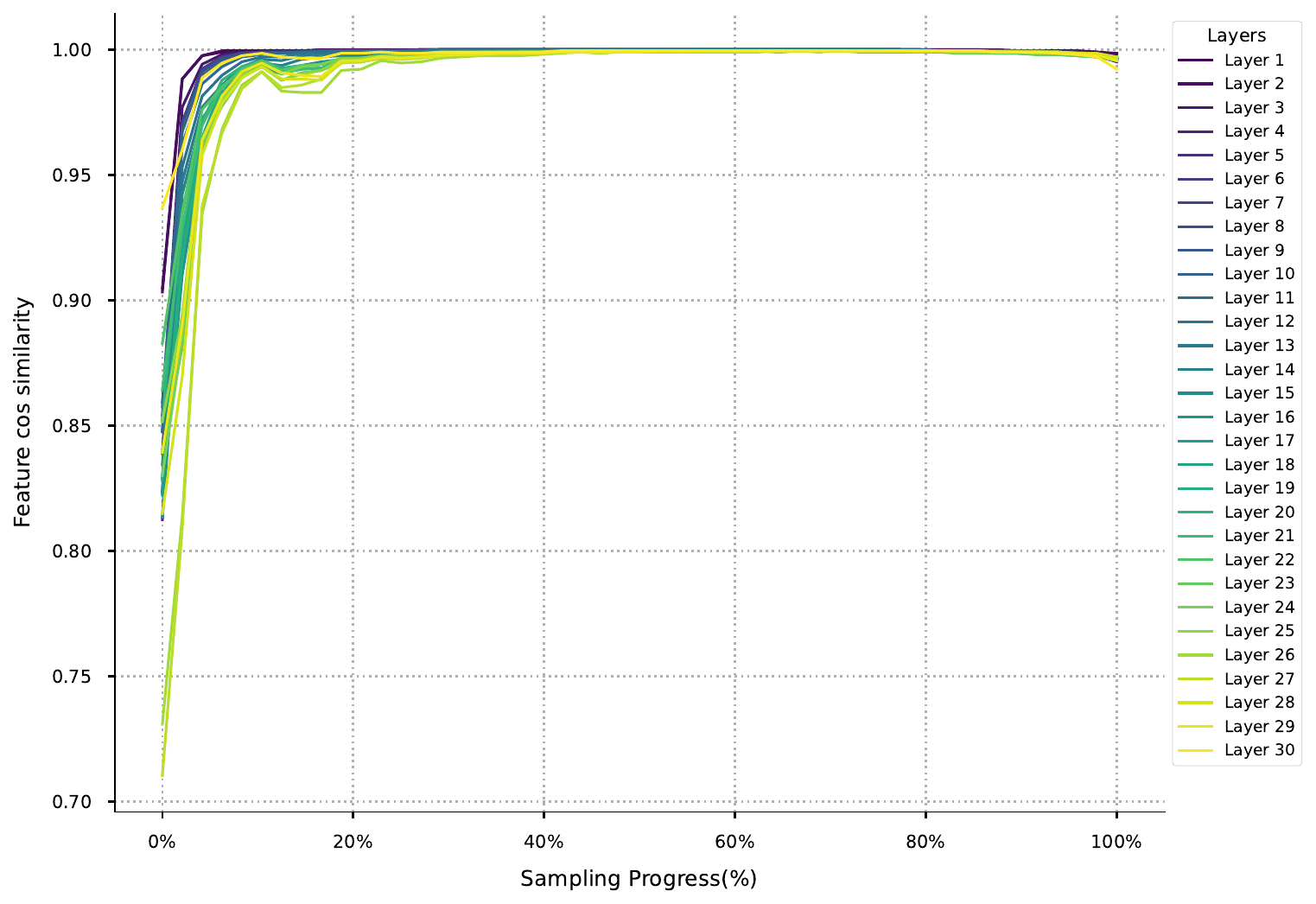}
      \caption{Directional consistency measured by cosine distance.}
      \label{fig:subfig2}
      \end{minipage}
      \caption{Analysis of U-type curve and directional consistency. Fig.\ref{fig:subfig1}: U-type similarity curve, demonstrating high similarity between feature maps of adjacent steps; Fig.\ref{fig:subfig2}: Directional consistency measured by cosine distance, validating the alignment of feature vectors.}
    \label{fig:compare}
  \end{figure}



\vspace{-5mm}
\subsection{U-Shaped Denoising Dynamics: An Adams-Bashforth Analysis of Diffusion Processes}
Diffusion models learn to generate data by iteratively denoising samples through a reverse process, which is mathematically described as a reverse-time SDE. This reverse SDE is associated with a deterministic Probability Flow Ordinary Differential Equation, whose exact solution, as shown in Eq. \eqref{eq:int}, can be derived analytically.
In this section, we leverage the Adams-Bashforth method to derive the relationship between the computational results at adjacent timesteps, effectively approximating future evaluations based on historical computations and further optimizing the inference efficiency. The integral of interest is defined as:
\begin{equation*}
    I = \int_{\lambda_s}^{\lambda_t} e^{-\lambda} \hat{\boldsymbol{\epsilon}}_\theta\left(\boldsymbol{x}_{t_\lambda(\lambda)}, t_\lambda(\lambda)\right) \mathrm{d} \lambda.
\end{equation*}
In order to approximate $I$ numerically, heuristically, we take the second-order Adams-Bashforth ($\operatorname{AB2}$) scheme to compute the vector-valued integral. Actually, we can get the following approximate relationship.
\begin{proposition}
     Let $h = \lambda_t - \lambda_s = \lambda_s -\lambda_o$ be step size, with some mild conditions the same as those in ~\citep{lu2022dpm}, we can obtain:
    \begin{equation} \label{eq:scale}
        \hat{\boldsymbol{\epsilon}}_\theta\left(\boldsymbol{x}_{t_\lambda(\lambda_s)}, t_\lambda(\lambda_s)\right)  =\frac{\alpha_t h}{3e^{-h}\alpha_t h -   2\sigma_{\lambda_t} e^{\lambda_o}(e^h-1) } \hat{\boldsymbol{\epsilon}}_\theta\left(\boldsymbol{x}_{t_\lambda(\lambda_o)}, t_\lambda(\lambda_o)\right) + O(h^2).
    \end{equation}
\end{proposition}

\begin{proof}
Let \(\boldsymbol{x}_{t}\) be a function satisfying the integral equation,
\begin{equation} \label{ode}
    \boldsymbol{x}_t = \frac{\alpha_t}{\alpha_s} \boldsymbol{x}_s - \alpha_t \int_{\lambda_s}^{\lambda_t} e^{-\lambda} \hat{\boldsymbol{\epsilon}}_\theta\left(\boldsymbol{x}_{t_\lambda(\lambda)}, t_\lambda(\lambda)\right) \mathrm{d} \lambda,  
\end{equation}
where \(\boldsymbol{x}_t, \boldsymbol{x}_s \in \mathbb{R}^n\) are vector-valued functions and \(\hat{\boldsymbol{\epsilon}}_\theta: \mathbb{R}^n \times \mathbb{R} \to \mathbb{R}^n\) is a twice continuously differentiable function parameterized by \(\theta\). We also assume \(t_\lambda(\lambda)\) is a continuous mapping from \(\lambda\) to time and $\lambda_t$ is continuously differentiable twice. With  Lagrange linear interpolation, we can approximate the integrand in equality \eqref{ode} with the following linear function
\begin{equation*}
    p(u) := (u+1)e^{-\lambda_s} \hat{\boldsymbol{\epsilon}}_\theta\left(\boldsymbol{x}_{t_\lambda(\lambda_s)}, t_\lambda(\lambda_s)\right) - u e^{-\lambda_o} \hat{\boldsymbol{\epsilon}}_\theta\left(\boldsymbol{x}_{t_\lambda(\lambda_o)}, t_\lambda(\lambda_o)\right),
\end{equation*}
where
\begin{equation*}
    u = \frac{\lambda-\lambda_s}{h} \in [0,1],
\end{equation*}
hence we are able to approximate $I$ with the value of $h \int_{0}^{1} p(u) \mathrm{d} u$, thus
\begin{equation} \label{eq:second-order}
     I = h\left(\frac{3}{2}e^{-\lambda_s} \hat{\boldsymbol{\epsilon}}_\theta\left(\boldsymbol{x}_{t_\lambda(\lambda_s)}, t_\lambda(\lambda_s)\right) - \frac{1}{2}e^{-\lambda_o} \hat{\boldsymbol{\epsilon}}_\theta\left(\boldsymbol{x}_{t_\lambda(\lambda_o)}, t_\lambda(\lambda_o)\right) \right) + O(h^3).
\end{equation}
With sufficiently small step size $h$, the integral $I$ is approximated with $O(h^2)$ global error. On the other hand, according to Equation \eqref{eq:first-order}, we observe that: 
\begin{equation*}
    \alpha_t I = \sigma_{\lambda_t}(e^h-1)\hat{\boldsymbol{\epsilon}}_\theta\left(\boldsymbol{x}_{t_\lambda(\lambda_s)}, t_\lambda(\lambda_s)\right) + O(h^2),
\end{equation*}
combining with Eq. \eqref{eq:second-order}, we can demonstrate the linear relationship of $\hat{\boldsymbol{\epsilon}}_\theta\left(\boldsymbol{x}_{t_\lambda(\lambda_o)}, t_\lambda(\lambda_o)\right)$ and $\hat{\boldsymbol{\epsilon}}_\theta\left(\boldsymbol{x}_{t_\lambda(\lambda_s)}, t_\lambda(\lambda_s)\right)$ as Equation \eqref{eq:scale} shows.
\end{proof}
Based on the above proposition, we find that when the truncation error is \(O(h^2)\), the network outputs of the diffusion model at adjacent time steps differ by a scaling factor, satisfying the linear relationship \eqref{eq:scale}.
The relationship between the network outputs at consecutive time steps can be expressed as \textit{coefficient}  \(r = \frac{\alpha_t h}{3e^{-h}\alpha_t h -   2\sigma_{\lambda_t} e^{\lambda_o}(e^h-1) } \), within the tolerance of second-order error. In other words, given the initial values, $r$ determines the trend of changes in the network outputs between adjacent time steps. 
The variation curve of the scale factor is shown in Fig. \ref{fig:scale}. As can be seen from the curve in the figure, the coefficient is around $1$ during the middle period, while it deviates from $1$ at the beginning and the end. This pattern aligns with the U-shaped curve commonly observed in~\citep{lv2024fastercache,ma2024deepcache}on caching-based accelerating diffusion model inference.

\begin{figure}[htbp]
    \centering
    \includegraphics[width=\textwidth]{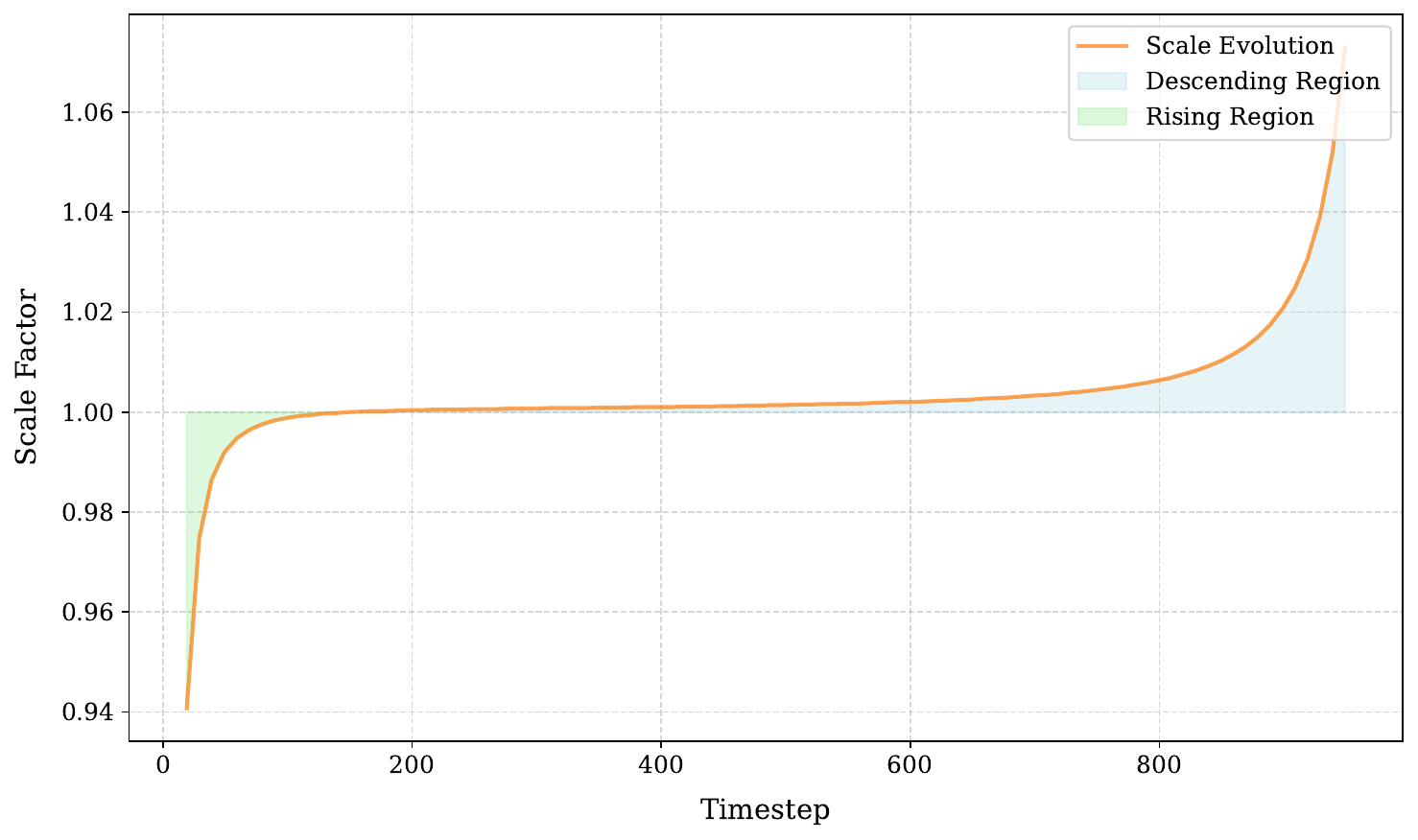} 
    \caption{Scale Factor Variation Curve with Timestep During the Diffusion Model Inference Process} 
    \label{fig:scale} 
\end{figure}

\subsection{AB-Cache: High-Order Linear Approximation Caching for Accelerated Adams-Bashforth Schemes} 


In this subsection, we introduce an innovative training-free caching acceleration method designed to minimize prediction inaccuracies caused by caching. By leveraging network output information from the preceding \( k \) time steps, this approach extends beyond traditional single-step caching strategies. The method's workflow is illustrated in Fig. \ref{fig:pipeline}. Unlike prior techniques that reuse only the immediately preceding time step's network output, we generalize the framework to incorporate information from the preceding \( k \) time steps while maintaining a truncation error bound of order \( O(h^k) \). To achieve this, we combine the \( k \)-th and \( (k+1) \)-th order Adams-Bashforth methods, leading to Proposition \ref{pro2}, which encapsulates the core innovation of our approach. By embedding this concept into a standard caching methodology, we derive a comprehensive workflow that balances computational efficiency with predictive accuracy, enabling more robust error control while retaining the efficiency benefits of caching.

\begin{proposition}\label{pro2}
    For $k=2,3,4$, we  can show the following $k$ linear recursive relationship of $\hat{\boldsymbol{\epsilon}}_\theta\left(\boldsymbol{x}_{t_\lambda(\lambda_{n})}, t_\lambda(\lambda_{n})\right)$ and $\hat{\boldsymbol{\epsilon}}_\theta\left(\boldsymbol{x}_{t_\lambda(\lambda_{n-i})}, t_\lambda(\lambda_{n-i})\right)$ where $1 \leq i \leq k$:  
    \begin{equation}    \label{k-recursive}            
        \hat{\boldsymbol{\epsilon}}_\theta\left(\boldsymbol{x}_{t_\lambda(\lambda_n)}, t_\lambda(\lambda_n)\right) = \sum_{i=1}^k (-1)^{i+1}\binom{k}{i}  e^{ih}\hat{\boldsymbol{\epsilon}}_\theta\left(\boldsymbol{x}_{t_\lambda(\lambda_{n-i})}, t_\lambda(\lambda_{n-i})\right) + O(h^k).
    \end{equation}
\end{proposition}

\begin{proof}
    According to the general $k$-th order Adams-Bashforth scheme in Appendix \ref{adams}, for the following integral,
    \begin{equation*}
        I = \int_{\lambda_{t_n}}^{\lambda_{t_{n+1}}} e^{-\lambda} \hat{\boldsymbol{\epsilon}}_\theta\left(\boldsymbol{x}_{t_\lambda(\lambda)}, t_\lambda(\lambda)\right) \mathrm{d} \lambda,
    \end{equation*}
    we respectively adopt $k$-th order and $(k+1)$-th order Adams-Bashforth scheme to approximate the value of $I$:
    \begin{equation} \label{k-appro}
         I = h\sum_{j=0}^{k-1} b^{(k)}_j e^{-\lambda_{n-j}}\hat{\boldsymbol{\epsilon}}_\theta\left(\boldsymbol{x}_{t_\lambda(\lambda_{n-j})}, t_\lambda(\lambda_{n-j})\right) + O(h^{k+1}) .
    \end{equation}
    \begin{equation} \label{k+1-appro}
         I = h\sum_{j=0}^{k} b^{(k+1)}_j e^{-\lambda_{n-j}}\hat{\boldsymbol{\epsilon}}_\theta\left(\boldsymbol{x}_{t_\lambda(\lambda_{n-j})}, t_\lambda(\lambda_{n-j})\right) + O(h^{k+2}) .
    \end{equation}
    where $b^{(k)}_j(0 \leq j \leq k-1)$ are the coefficients of $k$-th order scheme and $b^{(k+1)}_j(0 \leq j \leq k)$ are the coefficients of $(k+1)$-th order scheme. Combining the approximation relationship \eqref{k-appro} and \eqref{k+1-appro}, we can get the following linear recursive relationship for $k=1,2,3,4$ with the known weight coefficients in general Adams-Bashforth method
    \begin{equation*}
            \begin{aligned}
        e^{-\lambda_{n}}\hat{\boldsymbol{\epsilon}}_\theta\left(\boldsymbol{x}_{t_\lambda(\lambda_{n})}, t_\lambda(\lambda_{n})\right)=& \sum_{j=1}^{k-1} \frac{ b^{(k)}_j - b^{(k+1)}_j }{ b^{(k+1)}_0 - b^{(k)}_0 }  e^{-\lambda_{n-j}}\hat{\boldsymbol{\epsilon}}_\theta\left(\boldsymbol{x}_{t_\lambda(\lambda_{n-j})}, t_\lambda(\lambda_{n-j})\right)\\ 
        &- \frac{b^{(k+1)}_k}{ b^{(k)}_j - b^{(k+1)}_j }  e^{-\lambda_{n-k}}\hat{\boldsymbol{\epsilon}}_\theta\left(\boldsymbol{x}_{t_\lambda(\lambda_{n-k})}, t_\lambda(\lambda_{n-k})\right) + O(h^k).
    \end{aligned}
    \end{equation*}
    For $k=1,2,3,4$, it is easy to check the above relationship is the same as
    \begin{equation*}
        e^{-\lambda_n} \hat{\boldsymbol{\epsilon}}_\theta\left(\boldsymbol{x}_{t_\lambda(\lambda_n)}, t_\lambda(\lambda_n)\right) = \sum_{i=1}^k (-1)^{i+1}\binom{k}{i}  e^{-\lambda_{n-i}}\hat{\boldsymbol{\epsilon}}_\theta\left(\boldsymbol{x}_{t_\lambda(\lambda_{n-i})}, t_\lambda(\lambda_{n-i})\right) + O(h^k),
    \end{equation*}
    which is equivalent to the recursive relationship \eqref{k-recursive}.
\end{proof}
\begin{remark}
    For general \(k\), we conjecture that the linear recursive relationship \eqref{k-recursive} still holds for \(\hat{\boldsymbol{\epsilon}}_\theta\left(\boldsymbol{x}_{t_\lambda(\lambda_{n-i})}, t_\lambda(\lambda_{n-i})\right)\) where \(0 \leq i \leq k\). 
However, the general weight coefficients \(b^{(k)}_j\) (\(0 \leq j \leq k-1\)) depend on certain numerical integration (see Appendix A). While proving the general case is challenging, verifying specific instances is straightforward. In practice, we typically do not require higher-order approximations in experiments, so the formula in Proposition \ref{pro2} is already sufficient for practical purposes. For flow matching, just drop the exponential term in Eq. \eqref{k-recursive}.
\end{remark}

\section{Experiments}
\begin{figure*}[t] 
    \centering
    \includegraphics[width=\textwidth]{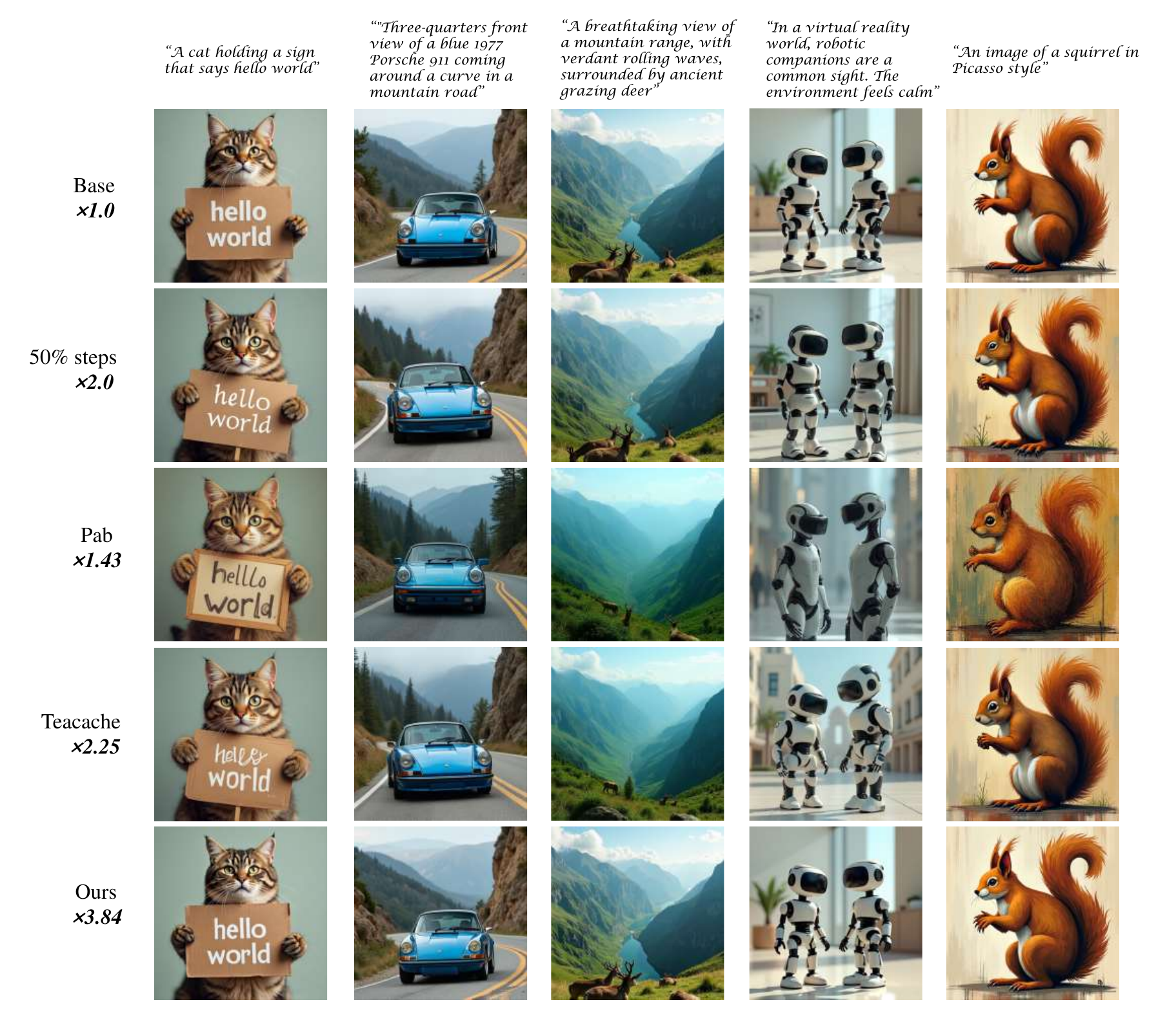}
    \caption{Visualization results for different acceleration methods on FLUX.1-dev model}
    \label{fig:image}
\end{figure*}
\subsection{Experimental Setup}
To comprehensively evaluate the effectiveness of our proposed acceleration method, we conduct extensive experiments across both image synthesis and video synthesis tasks, benchmarking against state-of-the-art models in each domain.

\textbf{Model Configurations.}
In our text-to-video generation study, we compare our approach with three state-of-the-art DiT-based model—HunyuanVideo~\citep{kong2024hunyuanvideo}, CogVideoX-5B~\citep{yang2024cogvideox}, and LTX-Video~\citep{hacohen2024ltx}. To ensure fair evaluation while respecting architectural differences, HunyuanVideo and LTX-Video employ the FlowMatchEulerDiscreteScheduler, whereas CogVideoX-5B utilizes the DDIM Scheduler. This configuration allows for systematic performance comparisons while accounting for each model's distinct scheduling mechanisms.
For text-to-image generation, we evaluate FLUX.1-dev using a 50-step FlowMatchEulerDiscreteScheduler as the default configuration.
All experiments are conducted with NVIDIA L40 48GB GPUs.



\begin{figure*}[t]  
    \centering
    \includegraphics[width=\textwidth]{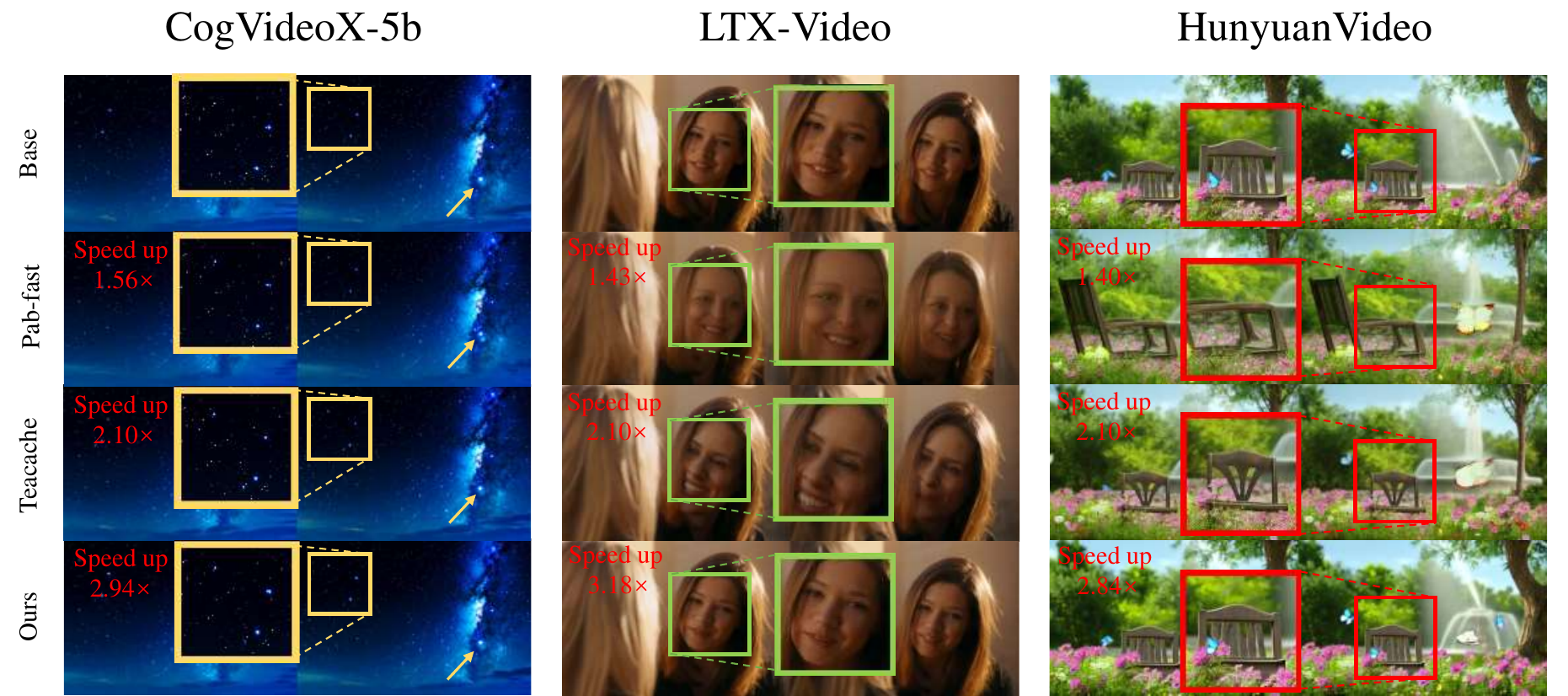}  
    \caption{Visual quality comparison between CogVideoX-5b, LTX and HunyuanVideo under different acceleration methods. 
maintaining high-quality generation without these problems.}  
    \label{fig:vc}
\end{figure*}

\begin{table*}[htbp]
\tiny
    \centering
    \caption{Quantitative evaluation of inference efficiency and visual quality in video generation models. AB-Cache consistently achieves superior efficiency and better visual quality across different base models, sampling schedulers, video resolutions, and lengths.}
    \label{tab: video}
    \begin{tabular}{c|ccc|cccc}
        \toprule
        \multirow{2}{*}{\textbf{Method}} & \multicolumn{3}{c|}{\textbf{Efficiency}} & \multicolumn{4}{c}{\textbf{Visual Quality}} \\ \cline{2-8}
         & \textbf{FLOPs (PFLOPs) $\downarrow$} & \textbf{Speedup $\uparrow$} & \textbf{Latency (s) $\downarrow$} & \textbf{VBench $\uparrow$} & \textbf{LPIPS $\downarrow$} & \textbf{SSIM $\uparrow$} & \textbf{PSNR $\uparrow$} \\

         \hline
        \hline
        \multicolumn{8}{c}{\textbf{LTX-Video} (48 frames, 704$\times$480)} \\
        \hline
        \rowcolor[gray]{0.9} LTX-Video $(T = 50)$ & 22.59  & 1$\times$ & 5.74 & 72.99\% & - & - & - \\
        $50\%$ steps & 11.30 & 2$\times$ & 2.87 & 72.86\% & 0.1985 & 0.8338 & 23.08 \\
        $\Delta$-DiT & 21.93 & 1.03$\times$ & 5.54 & 72.39\% & 0.6633 & 0.5448 & 12.68 \\
        PAB-slow & 16.49 & 1.37$\times$ & 4.19 & 72.42\% & 0.2885 & 0.7851 & 20.15 \\
        PAB-fast & 15.80 & 1.43$\times$ & 4.01 & 72.01\% & 0.3460 & 0.7445 & 18.92 \\
        Teacache & 10.76 & 2.10$\times$ & 2.73 & 72.51\% & 0.1166 & 0.8883 &  26.95 \\
        \hline
        AB-Cache (Order=3) & \textbf{7.10}  & \textbf{3.18 $\times$} & \textbf{1.80} & \textbf{73.26\%} & 0.1949 & 0.8321 & 22.91  \\
        AB-Cache (Order=2) & \textbf{7.10}  & \textbf{3.18 $\times$} & \textbf{1.80} & 73.18\% & \textbf{0.1286} & \textbf{0.8921} & \textbf{27.24}  \\
        \hline
        \hline
        \multicolumn{8}{c}{\textbf{HunyuanVideo} (61 frames, 320$\times$512)}  \\
        \hline
        \rowcolor[gray]{0.9} HunyuanVideo $(T = 50)$ & 61.32 & 1$\times$ & 150.04 & 79.37\% & - & - & - \\
        $50\%$ steps & 30.66 & 2.00$\times$ & 75.02 & 79.17\% & 0.1589 & 0.8077 & 24.11 \\
        $\Delta$-DiT & 59.53 & 1.03$\times$ & 145.67 & 77.21\% & 0.5692 & 0.4811 & 11.91 \\
        PAB-slow & 46.11 & 1.33$\times$ & 112.81  & 77.64\% & 0.2116 & 0.6923 &  20.98 \\
        PAB-fast & 43.80 & 1.40$\times$ & 107.17 & 76.95\% & 0.4094 & 0.5595 &  17.64 \\
        Teacache & 29.20 & 2.10$\times$ & 71.45 & 78.46\% & 0.2142 & 0.7622 &  21.78 \\
        \hline
        AB-Cache (Order=3) & \textbf{21.59} & \textbf{2.84$\times$} & \textbf{52.83} & 79.26\% & \textbf{0.1866} & \textbf{0.8251} & \textbf{25.78} \\

        \hline
        \hline
        \multicolumn{8}{c}{\textbf{CogVideoX-T2V-5B} (65 frames, 320$\times$512)} \\
        \hline
        \rowcolor[gray]{0.9} CogVideoX-T2V-5B $(T = 50)$ & 14.79 & 1$\times$ & 144.73  & 71.43\% & - & - & - \\
        $50\%$ steps & 7.40 & 2$\times$ & 72.37 & 71.22\% & 0.1583 & 0.8067 & 23.68 \\
        $\Delta$-DiT & 14.50 & 1.02$\times$ & 141.89 & 62.55\% & 0.5388 & 0.3736 & 13.85 \\
        PAB-slow & 10.88 & 1.36$\times$ & 106.42  &67.30\% & 0.3059 & 0.6550 &  18.80 \\
        PAB-fast & 9.48 & 1.56$\times$ & 92.78 & 66.81\% & 0.5499 & 0.4717 &  15.47 \\
        Teacache & 7.04 & 2.10$\times$ & 68.92 & 71.01\% & 0.1414 & 0.8269 & 24.18 \\
        
        \hline
        AB-Cache (Order=3) & \textbf{5.21} & \textbf{2.84$\times$} & \textbf{50.96} & 71.11\% & \textbf{0.0538} & \textbf{0.8326} & \textbf{25.30}  \\

        \bottomrule
    \end{tabular}
\end{table*}

\textbf{Evaluation Metrics.}
We evaluate the performance of our method from two key perspectives: \textbf{inference efficiency} and \textbf{generation quality}. For inference efficiency, we measure FLOPs to quantify computational complexity and latency (in seconds) to assess real-time inference speed. The evaluation methods for generation quality vary across different tasks, as detailed below:


\textbf{Text-to-video generation.}
VBench is a comprehensive benchmark suite designed to systematically evaluate the performance of video generative models across multiple dimensions, including visual quality, temporal coherence, and semantic alignment. We use VBench as the primary metric for evaluating video generation quality. Based on the 946 prompts provided by VBench, each prompt generates five independent videos, resulting in a total of 4,730 videos used to assess the 16 distinct metrics proposed by VBench. Additionally, we incorporate PSNR, LPIPS, and SSIM as supplementary metrics to further assess the quality of generated videos. To ensure a representative evaluation, we randomly sample 200 prompts from the set provided by VBench and use them to compute these metrics, allowing for a more comprehensive analysis of the video generation performance. 
These settings are consistent with the recent work~\citep{liu2024timestep,zhao2024real}, ensuring a fair and comparable evaluation of video generation quality.

\textbf{Text-to-image generation.}
We conduct inference on 200 DrawBench~\citep{saharia2022photorealistic} prompts to generate images at a resolution of $256 \times 256$. The generated samples are then evaluated using ImageReward~\citep{xu2023imagereward} and CLIP Score~\citep{hessel2021clipscore} as key metrics to assess image quality and text alignment.
Furthermore, we randomly select 30,000 captions from the COCO-2017 dataset [20] to generate 30,000 images with a resolution of $256 \times 256$ to compute the FID-30k metric. 

\subsection{Main Results}
\paragraph{\textbf{Quantitative Comparison on Text-to-Video Generation.}}
As shown in Table \ref{tab: video}, we compare our method with three training-free diffusion inference acceleration approaches—$\Delta$-DiT, PAB, and TeaCache—as well as a step reduction method, with the number of steps set to $50\%$. To verify the generalizability of our method, we conduct comparisons across multiple base models with different settings.
Based on the quantitative comparison results, our method achieves the highest PSNR, LPIPS, and SSIM scores while preserving video quality comparable to the base model without any acceleration techniques. Furthermore, according to VBench metrics, our method achieves the best performance, surpassing $\Delta$-DiT, PAB and Teacache. In evaluating the HunyuanVideo baseline, our method achieves a 2.84$\times$ speedup, demonstrating substantial acceleration. In the VBench evaluation, we observe an improvement in visual quality, though there is a slight decline in some dimension metrics. 
On CogVideoX, our method achieves a 2.84$\times$ speedup while maintaining the best PSNR, LPIPS, and SSIM scores. 
In the comparison with the LTX model, our acceleration method delivering the highest quality at a speedup of 3.18$\times$.

\paragraph{\textbf{Quality Comparison on Text-to-Video Generation.}}
Fig. \ref{fig:vc} compares the videos generated by our method with those produced by the original model, PAB, and Teacache, all of which operate with a lower speedup ratio. Under diverse prompts containing portraits, objects, and complex environments, our method demonstrates superior performance in: (i) preserving consistent facial details across frames, (ii) maintaining structural stability of foreground objects (e.g., chairs), and (iii) ensuring temporal coherence in background elements (e.g., star fields).
The results demonstrate that our generation method effectively maintains the quality and details of the output while achieving faster speed, outperforming the comparison methods.

\paragraph{\textbf{Quantitative Comparison on Text-to-image Generation.}}
In terms of generation quality, our method achieves the best $3.18\times$ speedup while maintaining the highest generation quality, as evidenced by the FID-30k score and ImageReward metric on the DrawBench, with almost no loss in quality compared to the original non-accelerated approach. 
Specifically, our method achieves a $3.18\times$ speedup while preserving superior generation quality, as evidenced by an ImageReward score of 0.7893, which exceeds the base model's score of 0.7717.
Regarding semantic alignment, our method achieves the best CLIP score under faster acceleration, demonstrating superior performance in alignment with the text prompts.

\begin{figure}[htbp]
    \centering
    \includegraphics[width=\linewidth]{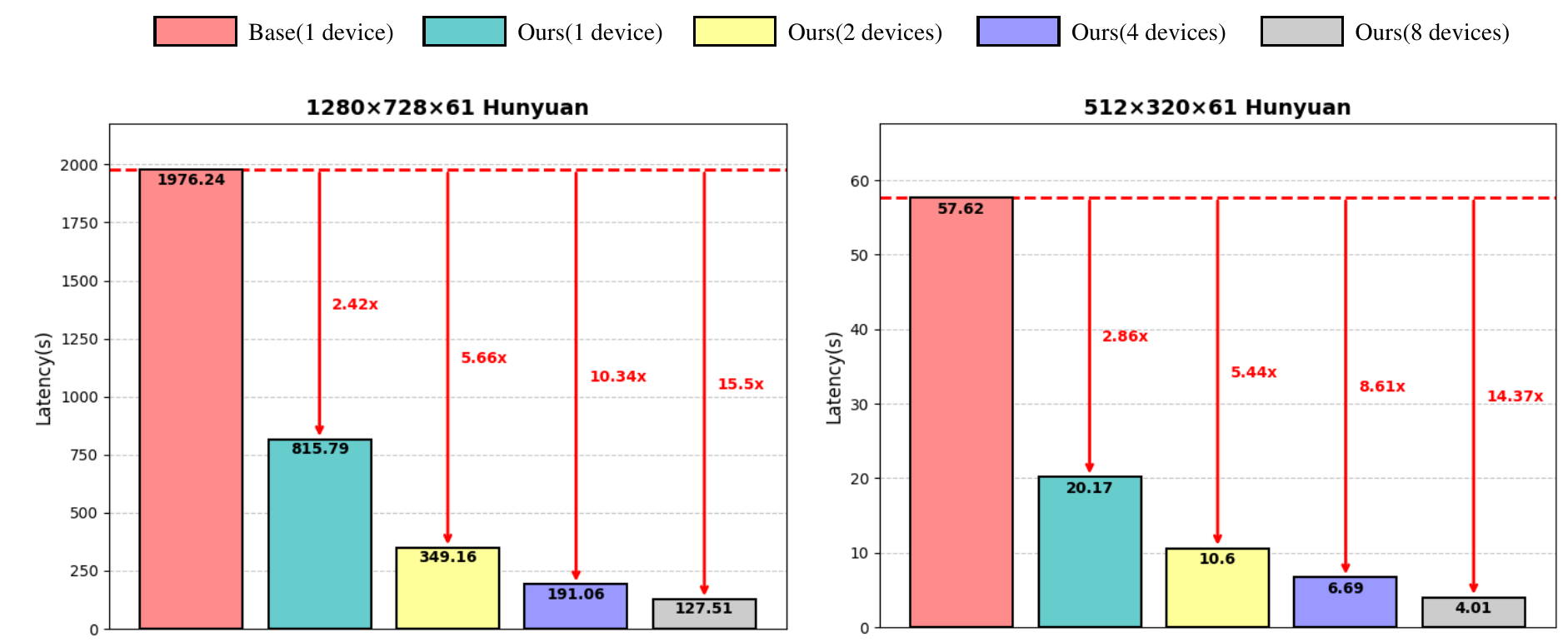} 
    \caption{Evaluation of acceleration and scalability of our method in generating larger videos under single- and multi-GPU configurations.}
    \label{fig:ab2}
\end{figure}
\textbf{Quality Comparison on Text-to-image Generation.}
In the qualitative comparison, we generate images using three representative prompts under identical settings across all methods. As shown in Fig.~\ref{fig:image}, our approach consistently produces the highest visual quality. The figure contrasts results from our method with those from the original model, PAB, and TeaCache, all of which exhibit lower speedup ratios.
Under diverse prompts involving text, vehicles, and natural landscapes, our method demonstrates superior performance in: (i) preserving the clarity and legibility of fonts, (ii) maintaining the structural consistency of objects, such as the front-facing orientation of the car, (iii) retaining fine-grained foreground details, such as the deer in the mountain scene, and (iv) ensuring rich, vibrant color fidelity throughout the generated images.
The results demonstrate that our generation method effectively maintains the quality and details of the output while achieving faster speed, outperforming the comparison methods.

\subsection{Ablation Studies}
We conduct ablation studies on the LTX model to evaluate the effectiveness of our proposed method.
Our ablation studies on the LTX model primarily investigate: (1) the impact of order (e.g., 1st-order vs. 3rd-order approximations) on generation quality and computational efficiency, (2) scaling behavior across varying frame lengths, resolutions, and GPU configurations. 
 
\textbf{Quality-Efficiency trade-off of Order.}
Table \ref{tab:order_comparison} demonstrates how varying approximation orders affect generation quality under the same experimental setup, consistent with our theoretical analysis that higher-order approximations reduce error. Importantly, we observe a substantial enhancement in aesthetic quality with increasing approximation order.

\textbf{Scaling Study.}
To validate the scalability of our method, we evaluated its performance on HunyuanVideo under different video lengths, resolutions, and GPU configurations. As shown in Figure \ref{fig:ab2}, our acceleration approach maintains consistent generation quality across varying computational scales and multi-GPU settings.
\begin{table}[htbp]
\centering
\caption{Performance comparison of different methods in terms of efficiency and visual quality.}
\label{tab:performance_comparison}
\renewcommand{\arraystretch}{1.5}
\setlength{\tabcolsep}{6pt}
\small
\resizebox{\linewidth}{!}{
\begin{tabular}{lcccccc}
\toprule
\multirow{2}{*}{\textbf{Method}} & \multicolumn{3}{c}{\textbf{Efficiency}} & \multicolumn{3}{c}{\textbf{Visual Quality}} \\ 
\cmidrule(lr){2-4} \cmidrule(lr){5-7}
& \textbf{Speedup$\uparrow$} & \textbf{Latency$\downarrow$} & \textbf{FLOPS$\downarrow$} (T) & \textbf{ImageReward$\uparrow$} & \textbf{CLIP$\uparrow$} & \textbf{FID$\downarrow$}  \\ 
\midrule
\rowcolor{gray!20} FLUX-1.dev & 1.00×  & 18.90  & 243.66  & 77.17  & 25.55  & 34.46   \\
50\%-step  & 2.00×  & 9.45  & 121.83  & 76.32  & 25.31  & 35.20   \\
PAB & 1.43×  & 13.22  & 170.39  & 72.39  & 25.59  & 34.14    \\
Teacache & 2.25×  & 8.40  & 108.29  & 75.34  & 25.66  & 32.85    \\
\hline
AB-Cache (Order=2) & \textbf{3.10×} & \textbf{6.10} & \textbf{78.6} & \textbf{78.93} & \textbf{25.73} & \textbf{32.12}    \\
\bottomrule
\end{tabular}
}
\end{table}

\begin{table}[htbp]
\centering
\caption{Comparison of different orders in terms of Quality Score, Semantic Score, and VBench Score.}  
\label{tab:order_comparison}
\resizebox{0.9\columnwidth}{!}{%
\begin{tabular}{lccc}
\hline
\textbf{Order}      & \textbf{Quality Score$\uparrow$} & \textbf{Semantic Score$\uparrow$} & \textbf{VBench Score$\uparrow$} \\ \hline
\rowcolor{gray!20} LTX-Video   & 78.29         & 51.78          & 72.99        \\
First-order & 77.29         & 48.35          & 71.50        \\
2nd-order   & 78.33         & 50.98          & 73.18        \\
3rd-order   & \textbf{78.79}         & 51.13        & \textbf{73.26}        \\ \hline
\end{tabular}%
}
\end{table}




\section{Conclusion}

In this work, we present a theoretically grounded acceleration method for diffusion models by leveraging the k-th order Adams-Bashforth numerical integration to establish a principled relationship between consecutive denoising steps. Our analysis reveals the U-shaped similarity pattern in adjacent outputs and provides a rigorous truncation error bound of $O(h^k)$, ensuring stable and high-quality generation.
The proposed high-order caching mechanism enables a nearly $3\times$ speedup in inference while maintaining performance parity with the base diffusion model, as validated across diverse architectures and tasks (image and video generation). Unlike heuristic reuse strategies, our method offers provable efficiency gains without compromising sample quality.
This work bridges the gap between theoretical analysis and practical acceleration in diffusion models, paving the way for real-time applications.

\bibliographystyle{ACM-Reference-Format}
\bibliography{reference}

\appendix
\clearpage 
\renewcommand{\theequation}{\thesection.\arabic{equation}}
\numberwithin{equation}{section}

\section{$k$-th Order Adams-Bashforth Scheme} \label{adams}

We can adopt the general $k$-th order Adams-Bashforth scheme to obtain a more accurate value of $I$. For simplicity, we use notation $t_k$ instead of $\lambda_{t_k}$ and $f(\tau):= e^{-\tau} \hat{\boldsymbol{\epsilon}}_\theta\left(\boldsymbol{x}_{t_\lambda(\tau)}, t_\lambda(\tau)\right)$, our goal is to get the integral of $f(\tau)$ over the interval \([t_n, t_{n+1}]\), thus:  
\begin{equation*}
    I =  \int_{t_n}^{t_{n+1}} f(\tau) \, d\tau.
\end{equation*}
According to the standrad $k$-th order Adams-Bashforth scheme, we need \(k\) previous time steps \(t_{n-k+1}, t_{n-k+2}, \dots, t_n\) (with uniform step size \(h\)).
We construct a \((k-1)\) degree Lagrange interpolating polynomial \(p(\tau)\) that satisfies:
\[
p(t_{n-j}) = f(t_{n-j}), \quad j = 0, 1, \dots, k-1.  
\]
Introduce the normalized variable \(s = \frac{\tau - t_n}{h}\), mapping the interpolation points to \(s = -(k-1), -(k-2), \cdots, 0\). The polynomial becomes:
\[
p(s) = \sum_{j=0}^{k-1} f(t_{n-j}) \cdot L_j(s),
\]
where \(L_j(s)\ (0 \leq j \leq k-1)\) are Lagrange basis polynomials:
\begin{equation*}
    L_j(s) = \prod_{\substack{m=0 \\ m \neq j}}^{k-1} \frac{s + m}{-j + m}.
\end{equation*}
Substitute \(p(\tau)\) into the integral and change variables to \(s\):
\[
\int_{t_n}^{t_{n+1}} p(\tau) \, d\tau = h \int_{0}^{1} p(s) \, ds = h \sum_{j=0}^{k-1} \left( f(t_{n-j}) \int_{0}^{1} L_j(s) \, ds \right).
\]
Define the integration weights \(b_j\) as:
\[
b_j = \int_{0}^{1} L_j(s) \, ds.
\]
The integral $I$ can be approximated by $h\sum_{j=0}^{k-1} b_j f(t_{n-j})$ with local truncation error $O(h^{k+1})$ and global truncation error $O(h^{k})$. The weights \(b_j\) depend only on \(k\) and are precomputed via integration of the basis polynomials. For example, for \(k=3\) ($3$rd order method), the basis polynomials are:
$$
L_0(s)=\frac{(s+1)(s+2)}{2}, \quad L_1(s)=-\frac{s(s+2)}{1}, \quad L_2(s)=\frac{s(s+1)}{2},
$$
and the coefficients for are derived as:
\[
b_0 = \frac{23}{12}, \quad b_1 = -\frac{4}{3}, \quad b_2 = \frac{5}{12}.
\]
Common coefficients \(k = 1\) to \(5\) are presented in Table \ref{coefficients}.

  \begin{table}[!htp]
    \caption{Weight Coefficients} \label{coefficients}
   \begin{spacing}{1.7}
    \centering
    \begin{tabular}{l|c|c|c|c|r}
        \hline
        Order $k$  & $b_0$ & $b_1$ & $b_2$ & $b_3$ & $b_4$ \\
        \hline
       \quad $1$ & $1$ &  &  & \\
        \hline
        \quad $2$ & $\frac{3}{2}$ & $-\frac{1}{2}$ & & &\\
        \hline
        \quad$3$ & $\frac{23}{12}$ & $-\frac{4}{3}$& $\frac{5}{12}$ & &\\
        \hline
        \quad $4$ & $\frac{55}{24}$ & $-\frac{59}{24}$ & $\frac{37}{24}$  & $-\frac{3}{8}  $ &\\
        \hline
        \quad $5$ & $\frac{1901}{720}$ & $-\frac{2774}{720}$ & $\frac{2616}{720}$  & $-\frac{1274}{720} $ & $\frac{251}{720}$ \ \\
        \hline
    \end{tabular}
   \end{spacing}
\end{table}

\end{document}